\documentclass{article}

\usepackage{PRIMEarxiv}

\usepackage[utf8]{inputenc} 
\usepackage[T1]{fontenc}    
\usepackage{hyperref}       
\usepackage{url}            
\usepackage{booktabs}       
\usepackage{amsfonts}       
\usepackage{nicefrac}       
\usepackage{microtype}      
\usepackage{lipsum}
\usepackage{fancyhdr}       
\usepackage{graphicx}       
\graphicspath{{media/}}     
\usepackage{chemformula}
\usepackage{amsthm}
\usepackage{subcaption}

\newtheorem{claim}{Claim}
\newtheorem{theorem}{Theorem}

\title{A Genetic Algorithm for Multi-Capacity Fixed-Charge Flow Network Design
\thanks{This research was funded by the U.S. Department of Energy's Fossil Energy Office through the Carbon Utilization and Storage Partnership (CUSP) for the Western USA (Award No. DE-FE0031837) as well as by the U.S. National Science Foundation through the Research Experience for Undergraduates program (Award No. 2243010)} 
}

\author{
  Caleb Eardley, Dalton Gomez, Ryan Dupuis \\
  School of Computing \\
  Montana State University \\
   \And
  Michael Papadopoulos \\
  Department of Computer Science \\
  Rensselaer Polytechnic Institute \\
   \And
  Sean Yaw \\
  School of Computing \\
  Montana State University \\
  \texttt{sean.yaw@montana.edu} \\
}

\begin{document}
\maketitle

\begin{abstract}
The Multi-Capacity Fixed-Charge Network Flow (MC-FCNF) problem, a generalization of the Fixed-Charge Network Flow problem, aims to assign capacities to edges in a flow network such that a target amount of flow can be hosted at minimum cost.
The cost model for both problems dictates that the fixed cost of an edge is incurred for any non-zero amount of flow hosted by that edge.
This problem naturally arises in many areas including infrastructure design, transportation, telecommunications, and supply chain management.
The MC-FCNF problem is NP-Hard, so solving large instances using exact techniques is impractical.
This paper presents a genetic algorithm designed to quickly find high-quality flow solutions to the MC-FCNF problem.
The genetic algorithm uses a novel solution representation scheme that eliminates the need to repair invalid flow solutions, which is an issue common to many other genetic algorithms for the MC-FCNF problem.
The genetic algorithm’s efficiency is displayed with an evaluation using real-world \ch{CO2} capture and storage infrastructure design data.
The evaluation results highlight the genetic algorithm's potential for solving large-scale network design problems.
\end{abstract}


\section{Introduction}\label{Sec:Intro}

The Multi-Capacity Fixed-Charge Network Flow (MC-FCNF) problem is a well-studied optimization problem encountered in many domains including infrastructure design, transportation, telecommunications, and supply chain management~\cite{ghost,antHybrid,whitman_scalable_2022}.
In the MC-FCNF problem, each edge in the network has multiple capacities available to it, with each capacity having its own fixed construction and variable utilization costs.
The objective of the MC-FCNF problem is to assign capacities to edges in the network such that a target flow amount can be hosted at minimal cost.
The MC-FCNF problem is a generalization of the Fixed-Charge Network Flow (FCNF) problem, which has a single capacity (and fixed and variable costs) available per edge.
The MC-FCNF problem is NP-Hard to approximate within the natural logarithm of the number of vertices in the graph~\cite{whitman_scalable_2022}.
As such, finding optimal solutions to large instances is often computationally infeasible.

Significant work has already been done on solving the MC-FCNF and FCNF problems using many techniques including mathematical programming, branch and bound, and optimal approaches~\cite{gendron_matheuristics_2018,Gendron14Branch,nahapetyan_adaptive_2008,eksioglu_solving_2003,diaby_successive_1991}. 
Multi-capacity edge networks are often referred to as buy-at-bulk network design problems, and are often framed as facility location problems, which is similar to the MC-FCNF problem but with added demand constraints on sinks~\cite{friggstad_lp-based_2019,chakrabarty_online_2018,arulselvan_exact_2017}. 
Genetic algorithms have also been introduced for variants of the MCNF problem~\cite{doostie2021novel,binh2014survivable,gaTransportation,onguetou2009solution,genetic,globalGenetic,xin2003physical,kwong2002use}.

In this paper, we introduce a novel genetic algorithm to solve the MC-FCNF problem.
The novel contribution of our genetic algorithm is the representation of a flow solution by an array of parameters that scale the fixed-costs for each edge in the network.
This representation ensures that each array corresponds to a valid flow, thereby eliminating the need for computationally expensive repair functions that are required by other genetic algorithms for the MC-FCNF problem~\cite{doostie2021novel,binh2014survivable,onguetou2009solution,jo_nonlinear_2007,globalGenetic,xin2003physical,kwong2002use}.
By avoiding costly repair functions, the proposed algorithm is able to efficiently find high-quality solutions to very large MC-FCNF problem instances.
The proposed genetic algorithm is inspired by slope scaling techniques previously employed for the FCNF problem~\cite{gendron_matheuristics_2018,crainic_slope_2004}.
It is a \textit{matheuristic}, as it employs mathematical programming to calculate a flow solution from a linear program parameterized with the fixed-cost scaling arrays~\cite{fischetti_matheuristics_2018}. 
Our genetic algorithm is similar to an algorithm proposed by \cite{cosma_efficient_2019}, though ours takes a different two-stage approach to handle multi-capacity edges.
Additionally, we provide more insight into the existence of the optimal solution in the search space.

An evaluation is presented that designs \ch{CO2} capture and storage (CCS) infrastructure deployments using real-world data composed of thousands of vertices and tens of thousands of edges.
In the evaluation, the genetic algorithm is compared to the the solution of an optimal integer linear program formulation of the MC-FCNF problem.
Results from the evaluation demonstrate the utility of the genetic algorithm for very large networks, even if the solution is very small compared to the full network.

The rest of this paper is organized as follows: Section~\ref{Sec:Problem} formally introduces the MC-FCNF program and formulates it as an integer linear program. 
Section~\ref{Sec:lp} presents a linear programming modification to the integer linear program that serves as the core to the genetic algorithm. 
Section~\ref{Sec:Alg} introduces the genetic algorithm and discusses the existence of the optimal solution in the search space.
Section~\ref{Sec:Eval} presents an evaluation of the genetic algorithm on real-world CCS data and the paper is concluded in Section~\ref{Sec:Conc}.

\section{Problem Formulation}
\label{Sec:Problem}

The MC-FCNF problem can be formulated as an integer linear program (ILP), as shown below:

\medbreak
\noindent
Instance Input Parameters:\\
\begin{tabular}{l l}
  \hspace{.5cm}$V$ & Vertex set\\
  \hspace{.5cm}$E$ & Edge set\\
  \hspace{.5cm}$K$ & Set of possible capacities for each edge\\
  \hspace{.5cm}$s \in V$ & Source vertex\\
  \hspace{.5cm}$t \in V$ & Sink vertex\\
  \hspace{.5cm}$c_{k}$ & Capacity of $k$\\
  \hspace{.5cm}$a_{ek}$ & Fixed construction cost of edge $e$ with capacity $k$\\
  \hspace{.5cm}$b_{ek}$ & Variable utilization cost of edge $e$ with capacity $k$\\
  \hspace{.5cm}$T$ & Target flow amount
\end{tabular}

\medbreak
\noindent
Decision Variables:\\
\begin{tabular}{l l}
  \hspace{.5cm}$y_{ek} \in \{0,1\}$ & Use indicator for edge $e$ with capacity $k$\\
  \hspace{.5cm}$f_{ek} \in \mathbb{R}^{\ge 0}$ & Amount of flow on edge $e$ with capacity $k$
\end{tabular}

\medbreak
\noindent
Objective Function:
\begin{equation}
    \label{eq:obj}
    \min \sum_{e \in E} \sum_{k \in K} \big(a_{ek} y_{ek} + b_{ek} f_{ek}\big)
\end{equation}

\medbreak
\noindent
Subject to the following constraints:
\begin{align}
\label{eq:Cap} &f_{ek}\le c_{k} y_{ek},\forall e \in E, k \in K\\
\label{eq:Con} &\sum_{\substack{e\in E:\\src(e)=v}} \sum_{k \in K} f_{ek} = \sum_{\substack{e'\in E:\\dest(e')=v}} \sum_{k \in K} f_{e'k}, \forall v \in V \setminus \{s,t\}\\
\label{eq:Tar} &\sum_{\substack{e\in E:\\src(e)=s}} \sum_{k \in K} f_{ek} = T
\end{align}

\noindent
Where constraint~\ref{eq:Cap} enforces the capacity of each edge and forces $y_{ek}$ to be set to one if $f_{ek}$ is non-zero.
Constraint~\ref{eq:Con} enforces conservation of flow at each internal vertex.
Constraint~\ref{eq:Tar} ensures that the total flow amount meets the target.

Since the MC-FCNF problem is NP-Hard, solving this ILP is intractable for large instances~\cite{whitman_scalable_2022}.
The objective of this paper is to introduce a novel algorithm that efficiently finds high-quality solutions to this ILP without directly solving it.

\section{Non-Integer Linear Program}
\label{Sec:lp}
In this section, we introduce a linear program (LP) that is a modification of the ILP presented in Section~\ref{Sec:Problem}.
Since it is an LP, this new formulation can be solved optimally in polynomial time.
This LP forms the foundation of the genetic algorithm discussed in Section~\ref{Sec:Alg}.
Two components of the ILP change to turn it into an appropriate LP:
\begin{enumerate}
    \item The binary decision variables $y_{ek}$ are removed, thereby turning the model into an LP. Since the $y_{ek}$ variables are removed, the fixed costs are then scaled and combined with the variable costs.
    \item A new scaling parameter, $d_{ek}$, is introduced for each capacity on each edge that will scale the fixed cost of the edge. These parameters form the representation of a flow solution in the genetic algorithm.
\end{enumerate}

Let $g_{ek}$ be the decision variable representing the amount of flow on edge $e$ with capacity $k$ in the LP, analogous to the $f_{ek}$ decision variable in the ILP.
Then, the objective function of the LP is:

\begin{equation}
    \label{eq:LPobj}
    \min \sum_{e \in E} \sum_{k \in K} 
    \big(\frac{a_{ek}}{d_{ek}}+ b_{ek}\big) g_{ek}
\end{equation}

\noindent
The constraints in the LP mirror the constraints in the ILP:
\begin{align}
\label{eq:CapLP} &g_{ek}\le c_k ,\forall e \in E, k \in K\\
\label{eq:ConLP} &\sum_{\substack{e\in E:\\src(e)=v}} \sum_{k \in K} g_{ek} = \sum_{\substack{e'\in E:\\dest(e')=v}} \sum_{k \in K} g_{e'k}, \forall v \in V \setminus \{s,t\}\\
\label{eq:TarLP} &\sum_{\substack{e\in E:\\src(e)=s}} \sum_{k \in K} g_{ek} = T
\end{align}

\noindent
Where constraint~\ref{eq:CapLP} enforces the capacity of each edge.
Constraint~\ref{eq:ConLP} enforces conservation of flow at each internal vertex.
Constraint~\ref{eq:TarLP} ensures that the total flow amount meets the target.

The output of this LP is a flow value for each $g_{ek}$.
Of course, the optimal flow found by the LP is likely not an optimal solution for the ILP.
The true cost of the LP's solution can be determined by calculating its value when input into the ILP's objective function (Equation~\ref{eq:obj}).
This is first done by defining an edge-use indicator function $z_{ek}$ and assigning it values as follows:
\begin{equation}
\label{eq:z}
z_{ek}=
\begin{cases}
    1,& \text{if } g_{ek} > 0\\
    0,& \text{if } g_{ek} = 0
\end{cases}
\end{equation}

This makes the true cost of the LP's solution equal to:
\begin{equation}
\label{eq:score}
\sum_{e \in E} \sum_{k \in K} \big(a_{ek} z_{ek} + b_{ek} g_{ek}\big)
\end{equation}

The genetic algorithm in Section~\ref{Sec:Alg} works by varying the $d_{ek}$ scaling parameters and scoring the resulting optimal $g_{ek}$ values found by the LP with Equation~\ref{eq:score}.

\section{Genetic Algorithm}\label{Sec:Alg}
Genetic algorithms are a common evolutionary heuristic method used for searching and optimization. 
Genetic algorithms manage a population of \textit{organisms} that each correspond to a solution to the problem.
The population evolves over iterations of the algorithm using evolutionary processes observed in nature including selection, crossover, and mutation operations.
Selection is the process of deciding which organisms of the population continue into the next iteration (i.e., next generation).
This is the mechanism that allows the algorithm to prioritize organisms that correspond to better solutions to the problem and control the size of the population.
Crossover is the generation of a new organism from two existing organisms, analogous to biological reproduction.
Similarly, mutation is the slight modification of an organism into a new one corresponding to a different solution.
Crossover and mutation operations are the mechanisms that allow the algorithm to search for new, and possibly better, solutions.

A number of genetic algorithms have been developed to solve various versions of the FCNF problem.
In these algorithms, organisms are broadly represented as either individual edges, or predefined routes through the network.
Representing organisms as individual edges typically involves a binary variable for each edge indicating its availability for use~\cite{doostie2021novel,xin2003physical}.
Alternatively, representing organisms as predefined routes involves a binary variable for each route in a set of predefined routes through the network~\cite{binh2014survivable,onguetou2009solution,kwong2002use}.
In the case of the individual edge representation, generating the initial population, crossover operations, and mutation operations often requires repairing the organism, as random sets of edges are unlikely to result in valid flows.
Using predefined routes simplifies repairing operations, but may still require repair in the event of capacity constraint violations, and is likely to result in sub-optimal solutions due to the limited set of routing options.
The genetic algorithms that use these representations address the issue by employing computationally expensive repair functions to make organisms correspond to valid flows.
Instead of representing an organism in this fashion, we represent it as an array of the fixed-cost scaling parameters $d_{ek}$ introduced in Section~\ref{Sec:lp}.
Then, the solution corresponding to this organism is the set of flow values $g_{ek}$ found by the LP from Section~\ref{Sec:lp}.
The result of this representation is that we can guarantee the solution corresponding to any organism is a valid flow, since the LP enforces that.
This avoids costly repair functions and is the key to the efficiency of our approach.

The motivation for using the fixed-cost scaling parameters as the organism representation is that it allows control over the amount of fixed-costs incurred, while also removing the integer variables from the optimal ILP.
As the scaling parameter decreases to zero, the scaled fixed-cost increases to infinity, thereby dissuading selection of that edge by the LP.
Conversely, as the scaling parameter increases to infinity, the scaled fixed-cost approaches zero, thereby encouraging selection of that edge by the LP.
The genetic algorithm is tasked with searching for an organism of scaling parameters whose corresponding flow solution is as close to optimal for the ILP as possible.
Given that the genetic algorithm is using the fixed-cost scaling parameters as a proxy for a solution instead of using a solution directly, an important question is whether or not there exists a set of scaling parameters that yields an optimal flow solution.
A proof that such a set of scaling parameters is guaranteed to exist is presented in Section~\ref{Sec:Exist}.
The key components and workflow of the genetic algorithm are described below:

{\bf Fitness Function.}
Genetic algorithms use fitness functions to rank and compare the population of organisms to aid the selection process.
Our fitness function first determines the flow solution for a given organism by solving the LP in Section~\ref{Sec:lp} to get the $g_{ek}$ flow values.
After the $g_{ek}$ values are determined, the solution's true cost in the context of the optimal ILP is calculated by Equation~\ref{eq:score}.
The output of Equation~\ref{eq:score} is used as the fitness of the organism, where lower values correspond to higher fitness.

{\bf Selection Function.}
To keep the size of the population computationally manageable, a selection function is employed to prune the population at each iteration. 
A selection function is also used to identify the organisms for crossover operations.
Our genetic algorithm implements a binary tournament selection function in an effort to prioritize high fitness organisms, while not ignoring all low fitness organisms. 
In binary tournament selection, two random organisms are selected, and the one with the higher fitness is kept, while the other is discarded. 
This ensures that high-fitness organisms are likely to remain in the population while maintaining the possibility for low-fitness ones to survive as well.
Binary tournament selection is repeated until the number of organisms in the population is at the desired size.

{\bf Crossover Function.}
A crossover function is used to generate a new child organism from two parent organisms already in the population. 
Our crossover function first randomly selects two parent organisms from the existing population. 
A child organism is constructed as a random interval of the $d_{ek}$ valued array from the first parent organism and the rest of the $d_{ek}$ valued array from the second parent organism.

{\bf Mutation Function.}
In order to mimic evolution and introduce another element of randomness into the search, a mutation function is used to further alter child organisms. 
After a child organism is generated with the crossover function, it may be randomly selected for mutation. 
During a mutation operation, a number of $d_{ek}$ values in the organism are selected and, with equal probability, either incremented up or down a random amount between zero and one. 
Mutated $d_{ek}$ values are not allowed to go below a lower bound to avoid negative values and divide by zero issues.

{\bf Genetic Algorithm.}
Using the functions described above, our algorithm operates as follows:
First, an initial population of organisms is randomly generated.
Each organism in the initial population is initialized as a $d_{ek}$ array filled with a random value between $\epsilon>0$ and the average value of the fixed-costs in the input instance. 
After the initial population of organisms are created, the algorithm proceeds in an iterative fashion.
At each iteration, the fitness of each organism is first calculated as described above. 
While the population size is less than some threshold, the crossover function is executed to generate child organisms.
The child organisms are also subject to randomized mutations from the mutation function.
Once the population has increased in size to the designated threshold, the selection function is run to reduce its size while statistically discarding the lower fitness organisms.
The algorithm keeps executing iterations until the running time reaches a designated time limit.

{\bf CPLEX Polishing.}
Once the time limit has been reached, the most fit organism's $g_{ek}$ flow values are used as a warm start for IBM's CPLEX optimization software. 
CPLEX polishing is run for one fifth of the total time limit resulting in the final flow values returned by the algorithm. 

\subsection{Optimal Search Viability}
\label{Sec:Exist}
The last thing that remains to be shown about the theoretical behavior of the genetic algorithm is that it is possible for it to find the optimal solution to the ILP.
This claim is not trivial, because the search space of the genetic algorithm is the set of possible $d_{ek}$ arrays, which are merely a proxy for flow values, the property we are actually optimizing for.
The motivation for formulating the LP in Section~\ref{Sec:lp} and representing the organisms in the genetic algorithm as $d_{ek}$ arrays follows from the following claim:
\begin{claim}
\label{cl:rationale}
    If each $d_{ek}$ equals the optimal ILP flow value for edge $e$ with capacity $k$, then the optimal flow found by the resulting LP is also an optimal flow for the ILP.
\end{claim}

The rationale for this claim is that if each $d_{ek}$ and $g_{ek}$ both equal the optimal ILP flow values, then the cost of the LP's objective function from Equation~\ref{eq:LPobj} will equal the optimal cost of the ILP's objective function from Equation~\ref{eq:obj}.
Claim~\ref{cl:rationale} is also the stated motivation behind other similar genetic algorithms for the FCNF problem~\cite{cosma_efficient_2019}.
However, this claim is false as can be seen in Figure~\ref{fig:counter} with the displayed fixed costs ($a_e$), variable costs ($b_e$), capacities ($c_e$), and a capture target of three.
In this instance, the optimal ILP solution is to set the amount of flow on $e_1$ and $e_3$ to two and the amount of flow on $e_2$ and $e_4$ to one for a total cost of $20$.
Setting $d_{e_1}$ and $d_{e_3}$ to two and $d_{e_2}$ and $d_{e_4}$ to one yields the LP objective of minimizing $7g_{e_1} + 6g_{e_2}$.
The optimal solution to this is to set the amount of flow on $e_1$ and $e_3$ to one and the amount of flow on $e_2$ and $e_4$ to two for a total cost of $21$, thereby contradicting Claim~\ref{cl:rationale}.

\begin{figure}
  \centering
  {\includegraphics[width=.45\textwidth]{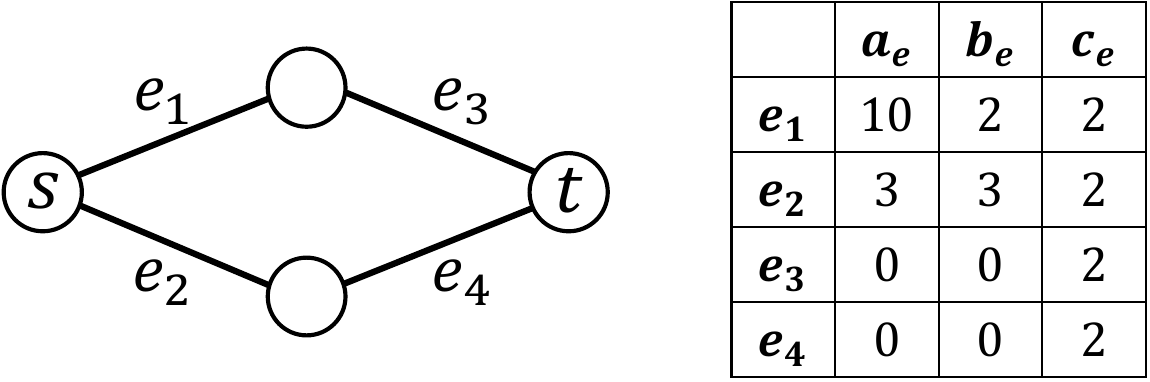}}
  \caption{Counterexample to Claim~\ref{cl:rationale} with the displayed fixed costs ($a_e$), variable costs ($b_e$), capacities ($c_e$), and a capture target of three. In this instance, the optimal ILP solution is $20$ with a flow of two units on $e_1$ and $e_3$ and one unit on $e_2$ and $e_4$. The corresponding optimal LP solution is $21$ with a flow of one unit on $e_1$ and $e_3$ and two units on $e_2$ and $e_4$.}
  \label{fig:counter}
\end{figure}

Since Claim~\ref{cl:rationale} is false, along with the fact that the $d_{ek}$ arrays are only proxies for the flow value solutions we seek, it remains to be shown that there actually exists a set of $d_{ek}$ values that will result in the genetic algorithm finding optimal flow values for the ILP.

\begin{theorem} \label{thm:solution}
For every problem instance, there exists a set of $d_{ek}$ values such that the optimal flow found by the resulting LP is also an optimal flow for the ILP.
\end{theorem}

\begin{proof}
\text{Let} $f_{ek}^{opt}$ be the optimal flow values found by the ILP and define the set of $d_{ek}$ values as follows:
\begin{equation}
\label{eq:d}
d_{ek}=
\begin{cases}
    \epsilon > 0,& \text{if } f_{ek}^{opt}=0\\
    \infty,& \text{if } f_{ek}^{opt}>0
\end{cases}
\end{equation}
When $d_{ek}$ is set to an $\epsilon$ value near zero, the scaled fixed cost $\frac{a_{ek}}{d_{ek}}$ makes those edges prohibitively expensive to include in LP solutions, so long as valid solutions exist that do not use those edges (such as the valid solution $f_{ek}^{opt}$).
Likewise, if $d_{ek}$ is set to a very large value, the scaled fixed cost is near zero.
These $d_{ek}$ values effectively restrict the LP to only selecting the edges and capacities with non-zero $f_{ek}^{opt}$ values.

Suppose that $H\subseteq E \times K$ is the set of edge-capacity pairs where $y_{ek}^{opt}$ equals one. 
Then, given the $d_{ek}$ values resulting from Equation~\ref{eq:d}, the objective for the LP becomes:
\begin{equation*}
    \sum_{e \in E} \sum_{k \in K} \left(\frac{a_{ek}}{d_{ek}} + b_{ek} \right) g_{ek} = \sum_{ek\in H} b_{ek} g_{ek}
\end{equation*}

Let $g_{ek}^{opt}$ be optimal flow values to the LP.
We aim to show that $g_{ek}^{opt}$ is also an optimal flow for the ILP.
First, $g_{ek}^{opt}$ is a valid solution to the ILP since $g_{ek}^{opt}$ is a valid flow of the ILP's target value on an identical graph with the same capacities.
Showing that $g_{ek}^{opt}$ is optimal for the ILP can be accomplished by using $g_{ek}^{opt}$ to feed the definitions for $z_{ek}$ in Equation~\ref{eq:z} and showing that:
\begin{equation*}
    \sum_{e \in E} \sum_{k \in K} \big(a_{ek} z_{ek}^{opt} + b_{ek} g_{ek}^{opt}\big) = \sum_{e \in E} \sum_{k \in K} \big(a_{ek} y_{ek}^{opt}+ b_{ek} f_{ek}^{opt}\big)
\end{equation*}

$z_{ek}^{opt}$ must equal one for all edges in $H$.
If $z_{ek}^{opt}$ equals zero for some edge in $H$, then the fixed costs incurred by $g_{ek}^{opt}$ are lower than the fixed costs incurred by $f_{ek}^{opt}$.
Also, since $g_{ek}^{opt}$ is optimal for the LP,
\begin{equation*}
\sum_{ek \in H} b_{ek} g_{ek}^{opt} \le \sum_{ek \in H} b_{ek} f_{ek}^{opt}
\end{equation*}
Thus, if $z_{ek}^{opt}$ equals zero for some edge in $H$, $g_{ek}^{opt}$ is a lower cost flow for the ILP than the optimal $f_{ek}^{opt}$, which is a contradiction.
Therefore, $f_{ek}^{opt}$ and $g_{ek}^{opt}$ must incur identical fixed costs and $z_{ek}^{opt}$ must equal one for all edges in $H$.

Suppose that,
\begin{equation*}
\sum_{ek \in H} b_{ek} g_{ek}^{opt} < \sum_{ek \in H} b_{ek} f_{ek}^{opt}
\end{equation*}

This implies that,
\begin{align*}
\sum_{ek \in H} a_{ek} + \sum_{ek \in H} b_{ek} g_{ek}^{opt} &< \sum_{ek \in H} a_{ek} + \sum_{ek \in H} b_{ek} f_{ek}^{opt}\\
\implies \sum_{ek \in H} \big( a_{ek} + b_{ek} g_{ek}^{opt}\big) &<  \sum_{ek \in H} \big( a_{ek} + b_{ek} f_{ek}^{opt}\big)\\
\implies \sum_{e \in E} \sum_{k \in K} \big( a_{ek} z_{ek}^{opt} + b_{ek} g_{ek}^{opt}\big) &<  \sum_{e \in E} \sum_{k \in K} \big( a_{ek} y_{ek}^{opt} + b_{ek} f_{ek}^{opt}\big)
\end{align*}
which is a contradiction, since $f_{ek}^{opt}$ is an optimal flow value for the ILP, and thus cannot be more expensive than the valid flow $g_{ek}^{opt}$.
Thus, 
\begin{equation*}
\sum_{ek \in H} b_{ek} g_{ek}^{opt} = \sum_{ek \in H} b_{ek} f_{ek}^{opt}
\end{equation*}

which implies that,
\begin{equation*}
\sum_{e \in E} \sum_{k \in K} \big(a_{ek} z_{ek}^{opt} + b_{ek} g_{ek}^{opt}\big) = \sum_{e \in E} \sum_{k \in K} \big(a_{ek} y_{ek}^{opt}+ b_{ek} f_{ek}^{opt}\big)
\end{equation*}
Therefore, defining the $d_{ek}$ values as in Equation~\ref{eq:d} yields an LP whose optimal flow values correspond to optimal flow values of the ILP.
\end{proof}

\section{Evaluation}\label{Sec:Eval}
To demonstrate the efficiency and effectiveness of the genetic algorithm presented in Section~\ref{Sec:Alg}, an evaluation was conducted using \ch{CO2} capture and storage (CCS) infrastructure design data.
CCS is a climate change mitigation strategy that involves capturing \ch{CO2} from industrial sources, notably power generation, transporting the \ch{CO2} in a pipeline network, and injecting it into geological reservoirs for long-term sequestration.
Large-scale CCS adoption will require the optimization of infrastructure for hundreds of sources and sinks and thousands of kilometers of pipelines.
The CCS infrastructure design problem aims to answer the question: 
What sources and sinks should be opened, and where should pipelines be deployed (and at what capacity) to process a defined amount of \ch{CO2} at minimum cost.
CCS sources and sinks both have fixed construction (or retrofit) costs, variable utilization costs, and capacities (or emission limits).
Pipelines have multiple capacities available, depending on the diameter of the pipeline installed.
Pipelines also have fixed construction costs and variable transportation costs that are dependent on the capacity selected.
Though the MC-FCNF problem has only a single source and sink, and does not have costs or capacities associated with nodes, the CCS infrastructure design problem can be translated into the MC-FCNF problem~\cite{olson2024planning,whitman_scalable_2022}.

The genetic algorithm was implemented and integrated into $SimCCS$, the Java-based CCS infrastructure optimization software, which uses CPLEX as its optimization model solver~\cite{middleton2020simccs}.
Initial performance simulations guided the parameterization of the genetic algorithm to have a population size of $10$, and mutation and crossover probability of both $50$\%.
A mutation probability of $50$\% means that each organism has a $50$\% chance of mutation, and a crossover probability of $50$\% means that $50$\% of the population is crossed over with a random other organism in each iteration of the genetic algorithm.
All reported genetic algorithm values are the average of three runs.
The optimal ILP from Section~\ref{Sec:Problem} was implemented in $SimCCS$ using CPLEX as well.
$SimCCS$ was used as a standardized way to represent CCS data and for problem and solution visualization.
Timing was coded directly into $SimCCS$ to ensure only the algorithm of interest was being timed during simulation.
Simulations were run on a machine with Ubuntu 20.04.5, an Intel Xeon W-2255 processor running at 3.7 GHz, and 64 GB of RAM. 
$SimCCS$ on this machine used IBM's CPLEX optimization tool, version 22.1.1.0.

\begin{figure}
    \centering
    \begin{subfigure}[b]{0.5\textwidth}
        \centering
        \includegraphics[width=0.9\textwidth]{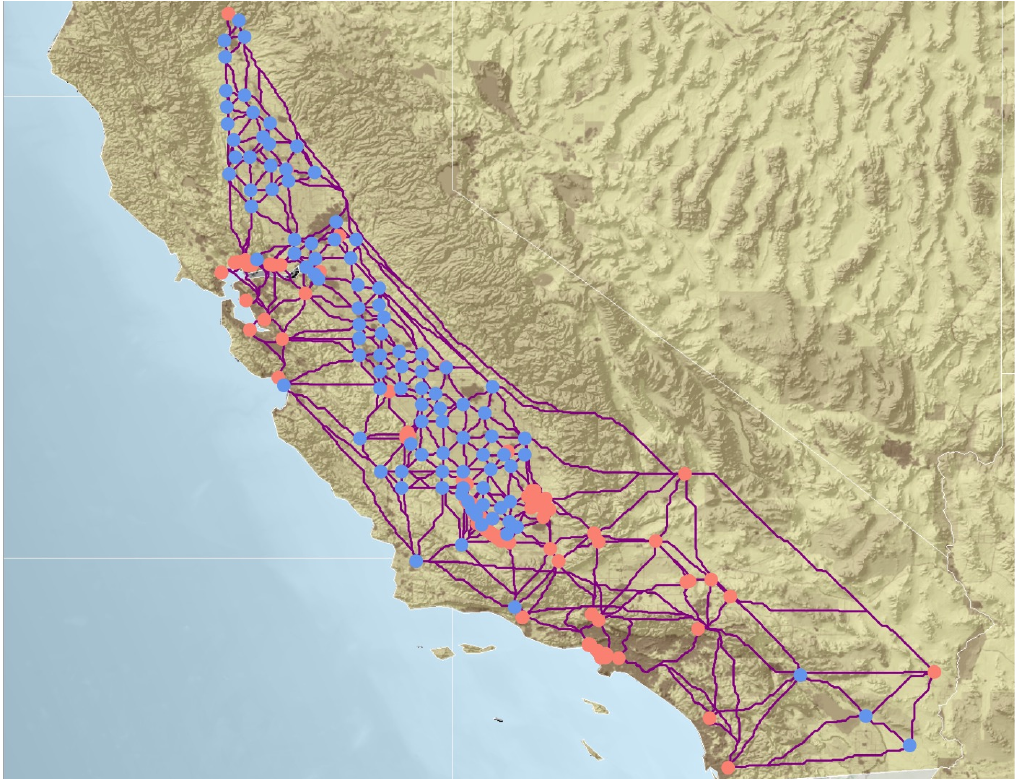}
        \caption{California dataset.}
        \label{fig:CAdata}
    \end{subfigure}%
    ~ 
    \begin{subfigure}[b]{0.5\textwidth}
        \centering
        \includegraphics[width=0.9\textwidth]{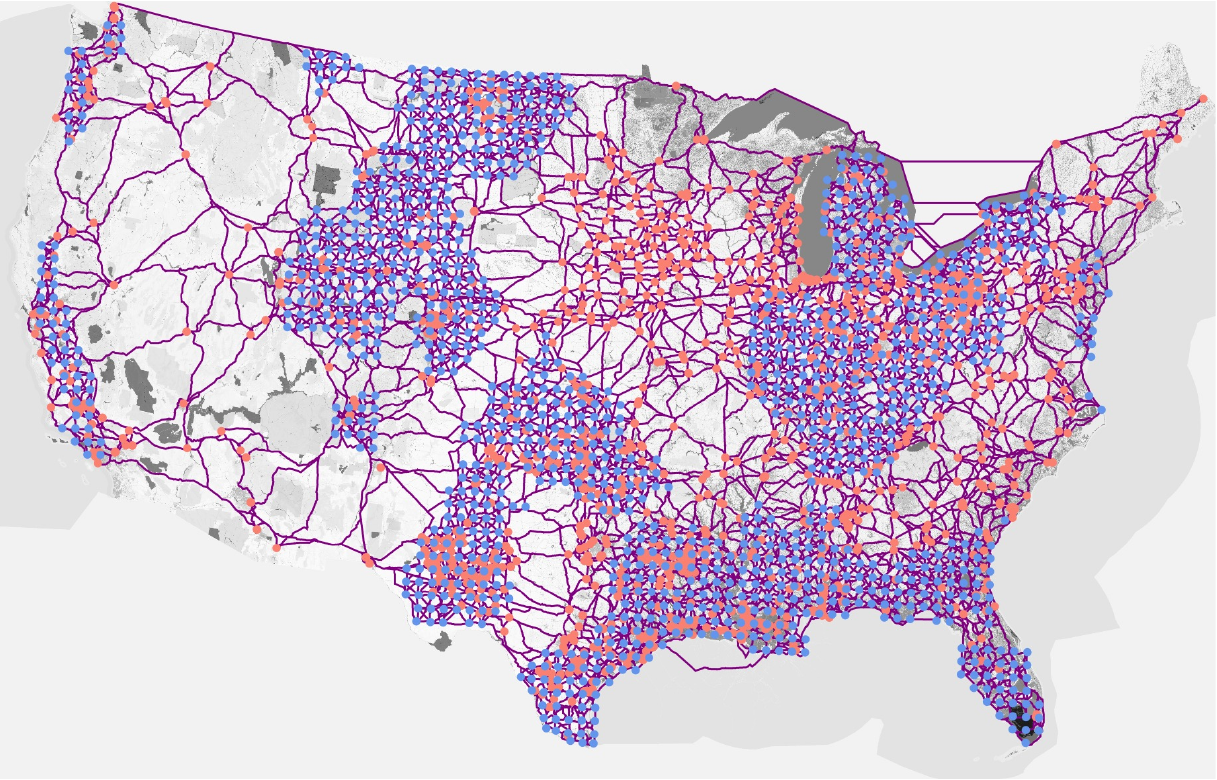}
        \caption{Continental United States dataset.}
        \label{fig:USdata}
    \end{subfigure}
    \caption{CCS datasets consisting of sources (red), sinks (blue), and possible pipeline routes.}
    \label{fig:data}
\end{figure}

The genetic algorithm was tested on two CCS infrastructure design datasets. 
The first dataset covers the United State's state of California and consists of $190$ sources with a total annual emission rate of $88.39$ Mt\ch{CO2}/yr, $102$ sinks with a total lifetime storage capacity of $37.18$ Gt\ch{CO2}, and $1188$ possible pipeline components (i.e., edges in the graph) with a total length of $17940.88$ KM and $11$ possible capacities on each edge. 
This data was collected as part of the US Department of Energy’s (DOE) Carbon Utilization and Storage Partnership project, one of the DOE’s Regional Initiatives to Accelerate CCS Deployment.
A map of this dataset is presented in Figure~\ref{fig:CAdata}.
The second dataset covers the continental United States and consists of $2746$ sources with a total annual emission rate of $532.61$ Mt\ch{CO2}/yr, $1202$ sinks with a total lifetime storage capacity of $2691.86$ Gt\ch{CO2}, and $22597$ possible pipeline components with a total length of $424674.41$ KM and $11$ possible capacities on each edge. 
This data was collected by Carbon Solutions, LLC as part of a study conducted by the Clean Air Task Force~\cite{cs,catf}.
Storage data was generated using the $SCO_2T$ geologic sequestration tool~\cite{ogland2023net}.
A map of this dataset is presented in Figure~\ref{fig:USdata}.
Candidate pipeline routes were generated in $SimCCS$ using its subset spanner network generation algorithm~\cite{yaw2019graph}.
The National Energy Technology Laboratory’s \ch{CO2} Transport Cost Model is used by $SimCCS$ to determine fixed construction and variable utilization costs for the $11$ discrete pipeline capacity options~\cite{netl18cost}.

To assess the efficiency of the genetic algorithm, its solution cost was compared to the solution cost found by CPLEX solving the optimal ILP, with both methods being allowed to run for set running time periods.
For the California dataset, those running time periods were $0.5$, $1$, $2$, $4$, and $8$ hours.
The target flow amount ($T$) was set to $80$ Mt\ch{CO2}/yr for all of the California scenarios.
For the continental United States dataset, the running time periods were $0.5$, $1$, $2$, $4$, $8$ and $16$ hours. 
The target flow amount was set to $500$ Mt\ch{CO2}/yr for the continental United States scenarios.
Figure~\ref{fig:CAtime} presents each algorithm's solution cost over the running time periods for the California dataset, and Figure~\ref{fig:UStime} presents the same results for the continental United States dataset.
The cost of the best solution found by the genetic algorithm in the California dataset was within $0.5$\% of the ILP's solution of across all running times.
Conversely, the cost of the best solution found by the genetic algorithm in the continental United States dataset was $17$\% lower than the ILP's solution after one hour, $7$\% lower after four hours, and $2$\% lower after $16$ hours.
This suggests that the genetic algorithm may have utility for very large problem instances.
The utility for small instances is likely limited, due to the speed of CPLEX.
Further, in applications that require rapid computation of MC-FCNF solutions, the genetic algorithm may exhibit beneficial performance for even smaller instances.

\begin{figure}
    \centering
    \begin{subfigure}[b]{0.5\textwidth}
        \centering
        \includegraphics[width=0.9\textwidth]{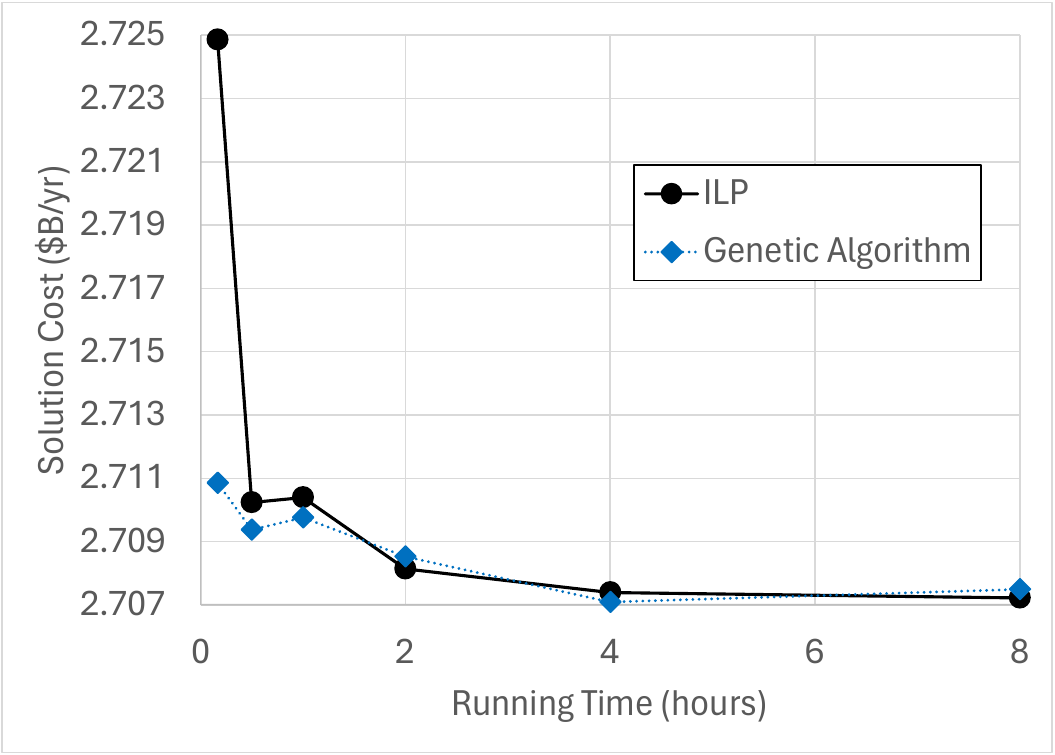}
        \caption{California dataset.}
        \label{fig:CAtime}
    \end{subfigure}%
    ~ 
    \begin{subfigure}[b]{0.5\textwidth}
        \centering
        \includegraphics[width=0.9\textwidth]{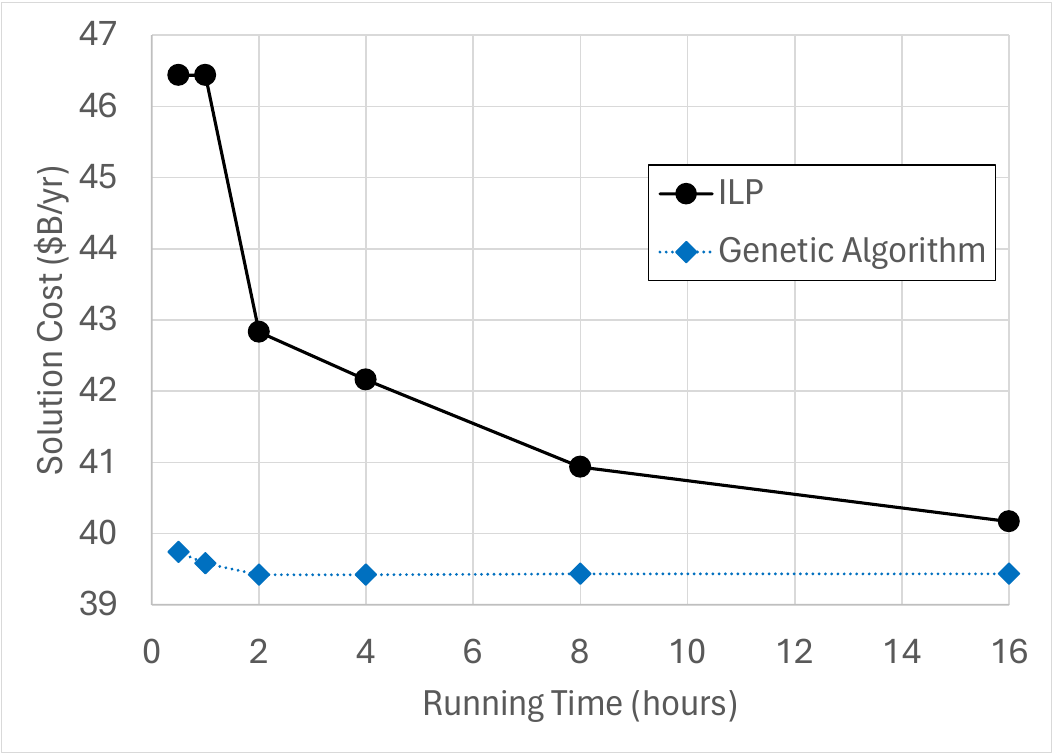}
        \caption{Continental United States dataset.}
        \label{fig:UStime}
    \end{subfigure}
    \caption{Solution cost versus running time for the genetic algorithm and optimal ILP.}
    \label{fig:time}
\end{figure}

Problem solvability is not only related to instance size, but also the amount of target flow being found.
To identify the impact that target flow amount has on solution quality, scenarios were run on the continental United States dataset where the target flow amount was varied from $1$ Mt\ch{CO2}/yr to $532$ Mt\ch{CO2}/yr.
The maximum annual capturable amount of \ch{CO2} for this dataset is $532.61$ Mt\ch{CO2}/yr.
Each algorithm was given two hours to solve each scenario.
Figure~\ref{fig:UStarget} presents each algorithm's solution cost for the various target flow amounts, as well as a table showing the specific solution cost values and the percent improvement of the genetic algorithm's solution over the ILP's solution.
Other than in very low and high target flow amount scenarios, the genetic algorithm was fairly consistent in its improvement over the ILP.
Interestingly, in the low target flow amount scenario, the genetic algorithm performed the best relative to the ILP.
This suggests that the genetic algorithm could have utility in a very large problem instance, even if the target flow amount is quite small.
This is likely a very realistic scenario, where a small solution is sought amongst a very large space of options, and provides a compelling argument for the utility of the genetic algorithm.

\begin{figure}
  \centering
  {\includegraphics[width=.99\textwidth]{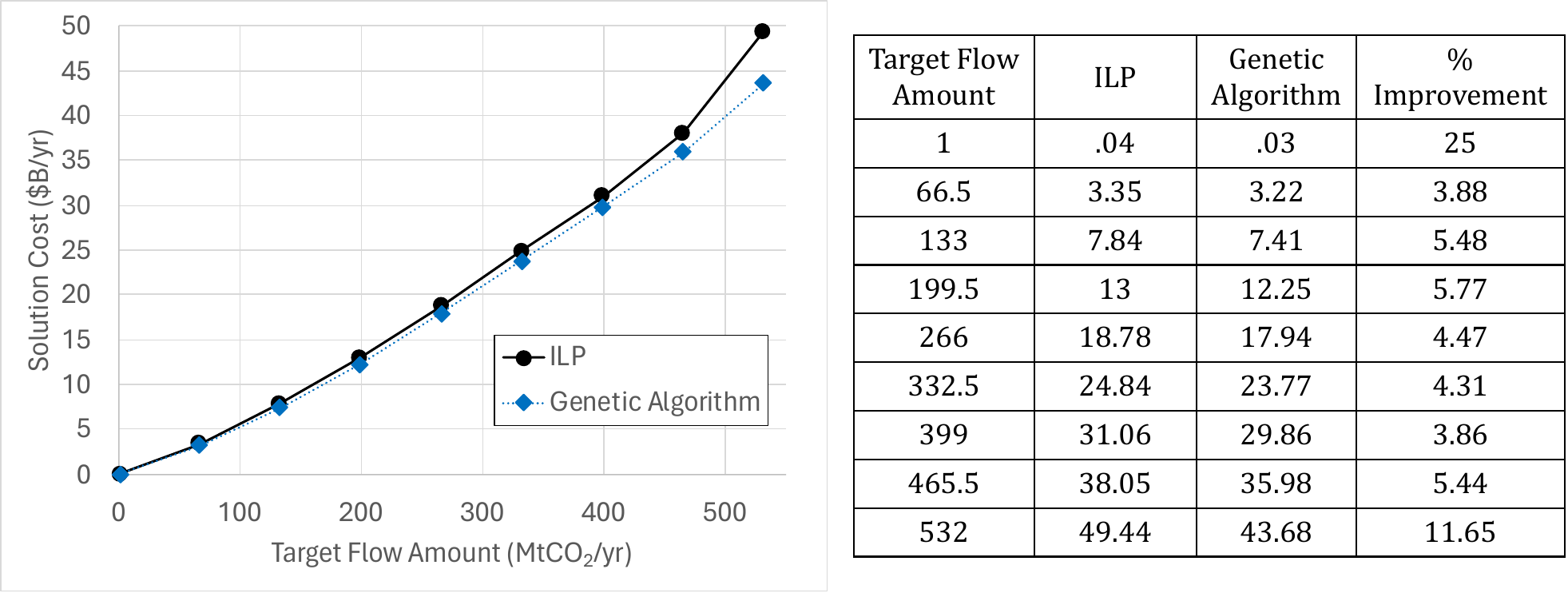}}
  \caption{Solution cost versus target flow amount for the genetic algorithm and optimal ILP.}
  \label{fig:UStarget}
\end{figure}

\section{Conclusion}\label{Sec:Conc}
In this paper, we addressed the MC-FCNF problem by formulating it as an ILP and proposing a novel genetic algorithm to find high-quality solutions efficiently. 
The key novel component of our approach is the use of fixed-cost scaling parameters as a proxy for direct flow values, allowing the genetic algorithm to search the solution space effectively without the need for computationally expensive repair functions.

Our genetic algorithm demonstrated significant efficiency and effectiveness in solving the MC-FCNF problem. 
By integrating the algorithm into the $SimCCS$ infrastructure optimization software, we were able to evaluate its performance on real-world CCS infrastructure design data. 
The results showed that the genetic algorithm consistently outperformed CPLEX solving an ILP on very large problem instances, and matched CPLEX's performance on moderately sized problem instances. 
The evaluation demonstrated the potential of the genetic algorithm in handling large and complex networks with varied target flow objectives, across a wide range of running time requirements. 

The genetic algorithm presented in this paper offers a robust and scalable solution to the MC-FCNF problem, providing an efficient alternative to traditional ILP solvers in realistic scenarios. 
Future work could involve tailoring the genetic algorithm to more specialized versions of the FCNF problem, including phased network deployments~\cite{jones2022designing}.
The fixed-cost scaling parameter technique may also prove useful in building evolutionary algorithms for other problems that can be modeled as network flow problems (e.g., facility location problems).
Future work could include implementing this genetic algorithm approach for use in facility location applications and testing it against benchmark datasets.
Future work could also include pursuing real-time applications where low computational running time is more critical than infrastructure design problems.
Finally, further exploring the performance of the genetic algorithm on very large instances with small solution sizes could reveal useful applications of the genetic algorithm to real-world problems.

\section*{Acknowledgments}
Source and storage data for the continental United Stated dataset was generously provided by Carbon Solutions, LLC~\cite{cs}.

\bibliographystyle{unsrt}  
\bibliography{references}

\end{document}